\documentclass[10pt, journal]{IEEEtran}
\IEEEoverridecommandlockouts
\usepackage{cite}
\usepackage{amsmath,amssymb,amsfonts}
\usepackage{amsthm}

\usepackage{algorithmic}
\usepackage{graphicx}
\usepackage{textcomp}
\usepackage{hyperref}
\usepackage{multirow}
\usepackage{xcolor}
\usepackage{algorithm}
\usepackage{adjustbox}
\usepackage{algorithmic}
\usepackage{pifont}
\usepackage[table,xcdraw]{xcolor}
\usepackage{bbding}
\usepackage{tikz}
\usepackage{xcolor}

\usetikzlibrary{shapes.geometric, arrows}
\def\BibTeX{{\rm B\kern-.05em{\sc i\kern-.025em b}\kern-.08em
    T\kern-.1667em\lower.7ex\hbox{E}\kern-.125emX}}
\begin{document}

\makeatletter
\newcommand{\linebreakand}{%
  \end{@IEEEauthorhalign}
  \hfill\mbox{}\par
  \mbox{}\hfill\begin{@IEEEauthorhalign}
}
\makeatother

\hyphenation{op-tical net-works semi-conduc-tor IEEE-Xplore}

\newtheorem{theorem}{Theorem}
\newtheorem{lemma}{Lemma}
\newtheorem{definition}{Definition}
\newtheorem{corollary}{Corollary}
\newtheorem{prop}{Proposition}
\newtheorem{assumption}{Assumption}

\title{Big2Small: A Unifying Neural Network Framework for Model Compression\\
}

\author{Jing-Xiao Liao, \textit{Member, IEEE}, Haoran Wang, Tao Li, Daoming Lyu, \\ Yi Zhang,~\textit{Senior Member, IEEE}, Chengjun Cai, \textit{Member, IEEE}, Feng-Lei Fan, \textit{Senior Member, IEEE}

\thanks{\textit{(Feng-Lei Fan is the corresponding author.)}}

\thanks{Jing-Xiao Liao, Haoran Wang, and Feng-Lei Fan are with Frontier of Artificial Networks (FAN) Lab, Department of Data Science, City University of Hong Kong, Kowloon, Hong Kong SAR, China (email: jingxiao.liao@cityu.edu.hk, fenglfan@cityu.edu.hk).}
\thanks{Tao Li is with Hong Kong Applied Science and Technology Research Institute (ASTRI) (email: litao@astri.org).}
\thanks{Daoming Lyu is with Sichuan Engineering Research Center for Big Data Visual Analytics, Sichuan University of Science and Engineering, Zigong, China (email: daoming.lyu@suse.edu.cn)}
\thanks{Yi Zhang is with the School of Cyber Science and Engineering, Sichuan University, Chengdu, China, and also with the Key Laboratory of Data Protection and Intelligent Management, Ministry of Education, Sichuan University, Chengdu, China (email: yzhang@scu.edu.cn)}
\thanks{Chengjun Cai is with Computer Science and Information Technology Center, City University of Hong Kong (Dongguan), Dongguan, China (email: chengjun.cai@cityu-dg.edu.cn)}
}








\maketitle

\begin{abstract}
With the development of foundational models, model compression has become a critical requirement. Various model compression approaches have been proposed such as low-rank decomposition, pruning, quantization, ergodic dynamic systems, and knowledge distillation, which are based on different heuristics. To elevate the field from fragmentation to a principled discipline, we construct a unifying mathematical framework for model compression grounded in measure theory. We further demonstrate that each model compression technique is mathematically equivalent to a neural network subject to a regularization. Building upon this mathematical and structural equivalence, we propose an experimentally-verified data-free model compression framework, termed \textit{Big2Small}, which translates Implicit Neural Representations (INRs) from data domain to the domain of network parameters. \textit{Big2Small} trains compact INRs to encode the weights of larger models and reconstructing the weights during inference. To enhance reconstruction fidelity, we introduce Outlier-Aware Preprocessing to handle extreme weight values and a Frequency-Aware Loss function to preserve high-frequency details. Experiments on image classification and segmentation demonstrate that \textit{Big2Small} achieves competitive accuracy and compression ratios compared to state-of-the-art baselines.

\end{abstract}



\section{Introduction}

\IEEEPARstart{W}{ith} the development of foundational models, serving a neural network in a resource-constrained setting becomes increasingly challenging. To solve this problem, model compression has emerged as a critical research direction, which reduces a network's size and computational complexity without much compromising its performance~\cite{zhu2024survey,11359594}. Model compression can be applied during or after model training, which are referred to as compression-aware training~\cite{baskin2021cat} and post-training compression~\cite{kwon2022fast}, respectively. Compared to compression-aware training, post-training compression is more favored by end users who neither afford the cost of model training nor have expertise on model training. To better serve end users, the holy grail of post-training compression is to realize so-called ``what you compress is what you train''. Currently, to the best of our knowledge, there are five important classes of post-training model compression methods: low-rank decomposition, quantization, pruning, ergodic dynamic systems, and knowledge distillation.

{\small$\bullet$} \textit{Low-rank decomposition}: Low-rank decomposition approximates weight matrices in dense and embedding layers with a product of smaller matrices~\cite{hsulanguage}. For instance, instead of storing a large weight matrix of size $m \times n$, we can factor it into two narrower matrices of sizes $m \times k$ and $k \times n$, where $k$ is much smaller than $m$ or $n$ ($k \ll \min(m,n)$). This restructuring dramatically reduces the number of parameters while preserving the network's ability to represent complex features, leading to a more compact and efficient model.

{\small$\bullet$} \textit{Quantization}: Quantization compresses a network by converting high-precision parameters and activation values into lower-bit representations~\cite{rokh2023comprehensive}. This technique systematically reduces the numerical precision of both weights (learned parameters) and activations (neuron outputs) through bit-width reduction, such as turning 32-bit floating-point numbers into efficient 8-bit integers. The process significantly decreases memory requirements and accelerates inference while maintaining functional accuracy.

{\small$\bullet$} \textit{Pruning}: Neural network pruning introduces sparsity by selectively eliminating less critical components—including synapses, neurons, layers, or entire architectural blocks~\cite{he2023structured}. Therefore, a plethora of research in this field focuses on developing specialized evaluation metrics to identify which elements of a network hold minimal importance.

{\small$\bullet$} \textit{Ergodic dynamic systems}: Ergodic dynamic systems use a low-dimensional dynamic system to approximate a high-dimensional space, thereby turning a high-dimensional parameter vector into a low-dimensional vector to realize the reduction in memory~\cite{fan2024hyper,cvitanovic1992circle}. Common dynamic systems are irrational windings, chaotic mapping, and so on.

{\small$\bullet$} \textit{Knowledge Distillation}: Knowledge distillation (KD) transfers learned capabilities from a large, complex model to a smaller, more efficient model~\cite{xu2024survey}. This technique trains a student model not only on original training data labels but also on the teacher's softened output probabilities (logits), which contain richer relational information about class similarities than hard labels alone. The resulting compressed model achieves comparable performance to its larger counterpart.

It can be seen that the current landscape of model compression remains to a large extent fragmented, with individual methods developed according to their own empirical heuristics. This fragmentation raises a fundamental question: \textit{Is it possible to unify disparate classes of model compression methods within a common mathematical framework?} Addressing this question is nontrivial and consequential. First, a unifying theory is intellectually compelling, as it would reframe existing techniques not as isolated tools but as distinct manifestations of a shared mathematical substrate. Second, such a framework would elevate the field from a collection of scattered heuristic techniques to a principled discipline, enabling the systematic design of compression methods, and facilitating the generation of new approaches from principles. The present work addresses this question through both theoretical derivation and empirical development.

Theoretically, we first establish a unifying mathematical framework for model compression grounded in measure theory, which we term the \textbf{universal compressibility theorem}. Formally, let $\Sigma$ denote the parameter set of the original model and $\Sigma^{\dagger}$ the parameter set of the compressed model. Model compression is then interpreted as constructing a mapping $g: \Sigma \to \Sigma^{\dagger}$ that reduces the set size, i.e., $\mathfrak{s}(\Sigma) > \mathfrak{s}(\Sigma^{\dagger})$, where the size functional $\mathfrak{s}(\cdot)$ is defined as $\mathfrak{s}(\cdot) = \texttt{Pointsize} \times \mathfrak{m}(\cdot)$. Here, $\mathfrak{m}(\cdot)$ denotes the measure of a set, capturing its cardinality in a continuous sense, and $\texttt{Pointsize}$ represents the average number of digits required to encode a point in that set. Additionally, the preimage $g^{-1}(\Sigma^{\dagger})$ is required to approximate $\Sigma$ within a prescribed error tolerance. Within this framework, we demonstrate that five families of existing compression techniques can be subsumed as special cases of the universal compressibility theorem. Building on this framework, we further instantiate the mapping $g$ as a neural network. This yields the \textbf{structural equivalence theorem}, which establishes that each model compression technique is mathematically equivalent to a neural network subject to a regularization, which offers a fresh perspective.

Empirically, since combining universal compressibility and  structural equivalence
theorems leads to a network approach, we materialize our theory by presenting \textit{Big2Small}, a data-free post-training model compression framework that translates the idea of implicit neural representations (INRs) from image compression to model compression. INRs excel at representing continuous signals with high fidelity and have proven effective across data denoising~\cite{haider2025inr}, compression~\cite{jayasundara2025sinr,dupont2021coin,dupontcoin++}, and super-resolution~\cite{tang2024stsr}. By treating weight tensors as discrete samples of an underlying continuous function, we reconstruct them using a small network. To minimize reconstruction error and preserve high-frequency components, we introduce outlier-aware preprocessing and a frequency-aware loss function. Notably, our approach can be orthogonally combined with other techniques (e.g., quantization) for even greater compression. Our main contributions are

\begin{itemize} 
\item We establish a unifying theoretical framework for model compression. We demonstrate that all existing major families of model compression algorithms can be unified into one mathematical framework, and each method possesses an equivalent neural network architecture subject to a regularization. To the best of our knowledge, it is the first time for the field of model compression to have this kind of unifying theory. 

\item We introduce \textit{Big2Small}, a data-free model compression paradigm that compresses large models using small INRs. To our knowledge, this work represents one of the first attempts to employ INRs for neural network compression, which extends the application boundary of INRs. We further propose several strategies, including Outlier-Aware Preprocessing and Frequency-Aware Loss, to improve the practical compression performance. 

\item Experiments on image classification and segmentation networks demonstrate that \textit{Big2Small} achieves competitive performance across multiple classic architectures compared to the state-of-the-art data-free post-training model compression baselines.
\end{itemize}

\section{Related Works}

\subsection{Model Compression Methods}
To put our work in perspective, this section briefly reviews recent advances of five major families of model compression methods: low-rank decomposition, quantization, model pruning, ergodic dynamic systems, and knowledge distillation.


\begin{table*}[ht]
\caption{Properties of model compression methods.}
\centering
\scalebox{0.95}{
\begin{tabular}{l|
>{\columncolor[HTML]{FFFFC7}}c 
>{\columncolor[HTML]{DAE8FC}}c 
>{\columncolor[HTML]{FFFFC7}}c 
>{\columncolor[HTML]{DAE8FC}}c 
>{\columncolor[HTML]{FFFFC7}}c }
\hline
Feature         & Low-Rank Decomposition                       & Quantization           & Pruning                                   & Ergodic Dynamic Systems                      & Knowledge Distillation                            \\ \hline
Primary Goal    & Matrix Factorization                         & Precision Reduction    & Redundancy Removal                        & Trajectory Representation                    & Behavior Mimicry                                  \\
Sparsity Type   & Structured                                   & Dense                  & Unstructured or Structured                & Continuous                                   & Dense                                             \\
Core Operations & SVD, CP, Tucker                              & Integer Mapping        & Sparsity Induction                        & Ergodic Hyperfunctions                       & Soft Target Matching                              \\
Data Reliance   & {\color[HTML]{FE0000} Low (Often Data-free)} & Moderate (Calibration) & {\color[HTML]{3166FF} High (Fine-tuning)} & {\color[HTML]{FE0000} Low (Direct Encoding)} & {\color[HTML]{3166FF} Very High (Teacher Output)} \\ \hline
\end{tabular}}
\label{tab:compare}
\end{table*}

{\small$\bullet$} \textit{Low-rank Decomposition}: Recent studies in low-rank decomposition have focused on automating the rank-selection process, which was traditionally a manual trial-and-error effort~\cite{hsulanguage}. Modern frameworks utilize sensitivity analysis to identify how much redundancy exists in each part of the network by measuring the Hessian trace of a layer~\cite{yang2024stable}. Recently, low-rank decomposition has been intensively used in large models. A representative work is LoRA, which freezes the pre-trained model weights and injects trainable rank decomposition matrices into each layer, reducing the number of trainable parameters for downstream tasks~\cite{hulora}. Moreover, TensorGPT is a training-free model compression approach based on the Tensor-Train Decomposition, which reduces the number of parameters by a factor of 2.31~\cite{xu2023tensorgpt}.

{\small$\bullet$} \textit{Quantization}: The frontier of quantization research is to push towards extreme low-bit regimes, giving rise to Binarized Neural Networks (BNNs)~\cite{hubara2016binarized} and Ternary Quantization~\cite{liu2023ternary}. BNNs represent weights and activations using just a single bit (e.g., $+1$ and $-1$), allowing for the replacement of expensive floating-point multiplications with bitwise XNOR and population count operation~\cite{hubara2016binarized}. Seminal works like XNOR-Net introduced the straight-through estimator to handle the non-differentiability of the sign function during backpropagation~\cite{rastegari2016xnor}. While extreme quantization often incurs a performance penalty on complex tasks like ImageNet classification, hybrid strategies like quantization-aware distillation have been developed, achieving high accuracy recovery even at 4-bit precision for massive models like Llama~\cite{kim2019qkd,liu2024llm}.

{\small$\bullet$} \textit{Pruning}: The pruning workflow typically involves an iterative process: a pre-trained model is evaluated to identify unimportant parameters to prune. Then, the model is fine-tuned to allow the remaining weights to compensate for the loss~\cite{he2023structured}. Recent innovations have introduced task-agnostic pruning methods and gradient-based importance measures, such as SparseGPT, which allow for the pruning of models with billions of parameters in a few hours without extensive retraining~\cite{frantar2023sparsegpt}. Global pruning techniques further enhance efficiency by applying a single threshold across the entire network, allowing the pruning rate to vary between layers based on their relative information density~\cite{bai2024sparsellm}. 

{\small$\bullet$} \textit{Ergodic Dynamic Systems}: Ergodic dynamic systems represent a novel approach in model compression research that redefines the problem as the issue of parsimonious parameter representation via a hyperfunction~\cite{fan2024hyper}. By identifying a suitable dynamic system, the weights of a large network can be encoded as a set of hyperparameters $\theta$ and a hyperfunction $h$, where each parameter $w_n$ is recovered via $w_n = h(\theta; n)$. Currently, this type of method has been successfully applied to compress a LLaMA series model~\cite{fan2024hyper}, Segment Anything Models (SAMs)~\cite{fan2025compress}, and large images~\cite{11359320}.

{\small$\bullet$} \textit{Knowledge Distillation}:
KD is highly effective for compressing LLMs. Recent advancements include ``online distillation'' methods, which train multiple student branches simultaneously without a static pre-trained teacher~\cite{guo2020online}; and ``dataset distillation'', which condenses massive training sets into compact synthetic datasets~\cite{fang2026knowledge}. While KD is highly effective at preserving emergent reasoning in LLMs, it can be computationally expensive to train the teacher-student pair, and the performance of the student is inherently capped by the capacity of the teacher and the compatibility of their architectures~\cite{xu2024survey}. 

While \textit{Big2Small} also relies on neural networks, it differs fundamentally from KD and self-KD in two key ways. Primarily, \textit{Big2Small} operates on a ``compression-decompression'' paradigm that dynamically reconstructs weights during inference, while KD only uses the student model after training. Furthermore, \textit{Big2Small} is applied layer by layer—using a separate network to learn each layer's weights—whereas KD optimizes a single student model to replicate the final predictions of the large model. Lastly, \textit{Big2Small} is a data-free compression method, while most KD methods rely heavily on data except for self-KD. Data-free model compression (DFMC) attempts to compress the model without accessing to the original training data, which poses privacy risks in sensitive domains such as healthcare and finance. Prevalent DFMC strategies primarily include pruning, quantization, model merging, and ergodic dynamic systems. For instance, unifying Data-Free Compression (UDFC) executes pruning and quantization jointly without requiring data access~\cite{bai2023unified}. Similarly, methods like ZipIt!~\cite{solodskikh2023integral} and Model Folding~\cite{wangforget} reduce model size by merging redundant computational units within the original networks. Data-Free Knowledge Distillation (DFKD) takes a different approach, employing synthetic data generated via GANs or optimized noise to facilitate student model fine-tuning~\cite{yin2020dreaming,raikwar2022discovering}. Lastly, previously mentioned Hyper-Compression is also the DFMC method~\cite{fan2024hyper,fan2025compress}. 

In summary, despite all model compression methods share the fundamental goal of reducing model size and computational cost~\cite{zhu2024survey}, they differ in mechanisms. A comparative overview of their key properties is provided in Table~\ref{tab:compare}. Building on this landscape, this paper introduces a unifying mathematical framework to provide a systematic and in-depth understanding of these diverse compression techniques, which is novel in the field of model compression.

\subsection{Implicit Neural Representation and Hypernetworks} 
Implicit Neural Representations (INRs) parameterize a signal as a continuous function that maps input coordinates to output signal values (e.g., RGB pixel intensities for images). To enhance the expressivity of INRs, recent studies have proposed specialized activation functions and architectural modules. For instance, sinusoidal (SIREN)~\cite{sitzmann2020implicit} and wavelet (WIRE)~\cite{saragadam2023wire} activation functions were designed to capture high-frequency details. Tancik \textit{et al.} introduced a positional encoding scheme that projects inputs into Fourier space, enabling the network to learn high-frequency functions within low-dimensional domains~\cite{Mildenhall2020}. Furthermore, Mehta \textit{et al.} proposed modulated periodic activations to strengthen local feature representation~\cite{Mehta_2021_ICCV}. Beyond architectural improvements, the flexibility of INRs allows for seamless integration with downstream compression techniques to further enhance compression ratios. These include quantization~\cite{dupontcoin++, gordon2023quantizing} and Bayesian coding approaches~\cite{guo2023compression,herecombiner}. 

The hypernetwork (often shortened to hypernet, \cite{stanley2009hypercube,chauhan2024brief}) is a small meta-network designed to generate the weights for a larger, primary target network. Instead of learning its parameters directly via gradient descent, the target network has its weights dynamically produced by the hypernetwork, which takes an input (such as an embedding or task ID) and outputs the weight matrices for the main model. This architecture allows for dynamic and adaptable model behavior, enables weight sharing across layers for efficiency, and is particularly useful in meta-learning scenarios where a model must quickly adapt to new tasks by having its weights generated on the fly.

Examples include one-shot generation~\cite{shamsian2021personalized,zhang2018graph}, component-wise generation~\cite{alaluf2022hyperstyle}, chunk-wise generation~\cite{chauhan2024dynamic}, and multiple generation~\cite{beck2023hypernetworks}. However, to the best of our knowledge, the idea of hypernet was not used for model compression before~\cite{chauhan2024brief}. On the one hand, the proposed \textit{Big2small} diverges fundamentally from the conventional paradigm of INRs in two key aspects. First, whereas INRs employ a neural network to directly represent data, \textit{Big2small} utilizes a network to generate the parameters of another target network. In this sense, \textit{Big2small} can be viewed as a generalization of the INR concept, extending its application from data to network. Second, two approaches operate in opposite scaling regimes relative to their targets. INRs often employ networks larger than the data itself to achieve higher fidelity, while \textit{Big2small} is explicitly designed for compression, producing a compact representation of a larger target network. On the other hand, the proposed \textit{Big2small} is also different from the idea of hypernet. The former is for static compression of the target network, while the latter is for the dynamic tuning of the target network.

\section{A unifying Theoretical Framework of Model Compression}

In this section, we propose a unifying mathematical definition of model compression as a {neural function approximation} problem within a measure-theoretic framework. Our overarching insight is to interpret model compression as the projection from a large collection of parameters to a small collection of parameters. Under this umbrella, we present two primary theoretical results: i) \textbf{universal compressibility theorem} proving that these compression methods (low-rank, pruning, quantization, ergodic dynamic systems, knowledge distillation) can rigorously restore the original weights by the inverse projection with an error $\epsilon$,  which warrants compressibility; and ii) \textbf{structural equivalence theorem} demonstrating that all five techniques are mathematically equivalent to specific architecture of neural networks.

\subsection{Preliminaries}

We begin by establishing the geometric foundations of our framework, distinguishing between the original parameter space and the bounded parameter sets.

\begin{definition}[Parameter Space and Sets]
    Let $\Theta \cong \mathbb{R}^D$ be a finite-dimensional Euclidean space equipped with the $\ell_2$ norm $\|\cdot\|_2$, representing the unconstrained parameter space. i) We define the original parameter set $\Sigma \subset \Theta$ as a compact subset of weights, satisfying $\|\theta^*\|_2 \le R$ for all $\theta^* \in \Sigma$. ii) The compressed weight is defined as ${\theta}^{\dagger} \in \Sigma^{\dagger}$. iii) The reconstructed weight of the compressed model is defined as $\hat{\theta} \in \hat{\Sigma}$. Let $\mathfrak{m}$ denote the standard Lebesgue measure on $\Theta$, restricted to $\Sigma$, such that $\mathfrak{m}(\Sigma) > 0$.
\end{definition}

\begin{definition}[Distance Function]
Define a distance function as $h(\theta^*, {\theta}^{\dagger})$, we utilize the Euclidean distance to quantify the approximation error between the target weight $\theta^*$ and the compressed weight ${\theta}^{\dagger}$. 

\begin{enumerate}
    \item In the case of low-rank decomposition, quantization, pruning, and ergodic systems, the distance function is
\begin{equation}
    \texttt{dist}(\theta^*, \theta^{\dagger}) = \|\theta^*-g^{-1}(\theta^{\dagger})\|_2,
\end{equation}
where $g^{-1}$ recovers the target weights from the compressed weight. 
\item In the case of knowledge distillation, the distance function is
\begin{equation}
    \texttt{dist}(\theta^*, \theta^{\dagger}) = \|F^*(\theta^*,x)-F^\dagger(\theta^\dagger,x)\|_2,
\label{eq:dist2}
\end{equation}
where $F^*(\theta^*,x)$ and $F^\dagger(\theta^\dagger,x)$ are the teacher network's and the student network's output with respect to $x$.
\end{enumerate}

\end{definition}

\begin{definition}[Sets Size]
    Let $\texttt{Pointsize}(\cdot)$ denote the average number of bits required to encode a point in a given set. We define the size functional of a parameter set as
    \begin{equation}
        \mathfrak{s}(\Sigma) = \texttt{Pointsize}(\Sigma) \times \mathfrak{m}(\Sigma).
    \end{equation}

\end{definition}

\begin{definition}[Model Compression Mapping]
   We define the mapping function as $g$. For the target optimal weights $\theta^* \in \Sigma$, $g$ compresses the target weights to 
   \begin{equation}
       {\theta}^\dagger = g(\theta^*) \in \Sigma^{\dagger},
   \end{equation}
   and reconstructs the target weights
   \begin{equation}
       \hat{{\theta}}= g^{-1}(\theta^\dagger) \in \hat{\Sigma}.
   \end{equation}
The reconstructed weights  must be bounded by a given error tolerance $\epsilon>0$:
        \begin{equation}
           \texttt{dist}(\theta^*,\theta^{\dagger}) < \epsilon.
        \end{equation}
Here, we slightly abuse the notation. While $g$ and $g^{-1}$ denote the compression and reconstruction mappings respectively, they do not form a strict mathematical inverse pair due to the lossy nature of compression.
\end{definition}


\begin{definition}[Sets of Compression Methods]
    We define the compressed parameter sets $\Sigma^{\dagger}$ obtained from five primary compression paradigms. 
    \begin{enumerate}
        \item \textbf{Low-Rank Set ($\Sigma^{\dagger}_r$):} The set of weight matrices with a bounded $\texttt{rank}(\cdot)$. 
        \begin{equation}
            \Sigma^{\dagger}_r = \{ \theta^{\dagger} \in \mathbb{R}^{n \times m} \mid \texttt{rank}(\theta^{\dagger}) \le r \}.
        \end{equation}
        \item \textbf{Quantization Set ($\Sigma^{\dagger}_q$):} The uniform lattice defined by a bit-width $b$. 
        \begin{equation}
        \Sigma^{\dagger}_q = \{ \theta^{\dagger} \in \mathbb{R}^D \mid \theta^{\dagger}_i \in \{-R + j\Delta \mid j \in \{0, 1, \dots, 2^b-1\} \} \}.
        \end{equation}
        \item \textbf{Pruning Set ($\Sigma^{\dagger}_p$):} The set of subspaces with at most $N$ non-zero elements. 
        \begin{equation}
            \Sigma^{\dagger}_p = \{ \theta^{\dagger} \in \mathbb{R}^D \mid \|\theta^{\dagger}\|_0 \le N\}.
        \end{equation}
        \item \textbf{Ergodic Set ($\Sigma^{\dagger}_{e}$):} The set of trajectories generated by a chaotic map $T$. 
        \begin{equation}
            \Sigma^{\dagger}_{e} = \{ k \in \mathbb{N} \mid \hat{\theta} = T^k(\theta_0^\dagger), \, \theta_0^\dagger \in  \mathbb{R}^D \}.
        \end{equation}
        \item \textbf{Knowledge Distillation Set ($\Sigma^{\dagger}_{k}$):} The parameter set represents a smaller ``student'' architecture. 
        \begin{equation}
            \Sigma^{\dagger}_{k} = \{ \theta^{\dagger} \in \mathbb{R}^d \times \{0\}^{D-d}\}.
        \end{equation}
        
    \end{enumerate}
\end{definition}

\subsection{Theory I: Universal Compressibility}

In this section, we establish the theoretical validity of model compression. We propose a unifying theorem that shows that the compression sets ($\Sigma^{\dagger}$)  defined by five primary compression methods rigorously reduce the set size $\mathfrak{s}(\cdot)$ of the original set, and approximating the original parameters by the reconstructed parameters is bounded by a given error
tolerance  $\epsilon$. An illustration framework of this theory is depicted in Figure~\ref{fig:theory1}.

\begin{figure}[ht]
    \centering
    \includegraphics[]{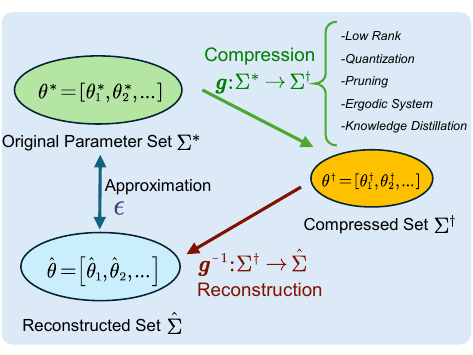}
    \caption{The theoretical framework of Theorem 1.}
    \label{fig:theory1}
\end{figure}

\begin{theorem}[Universal Compressibility]
    Let the original weight parameter set be $\Sigma$ with $\mathfrak{m}(\Sigma) > 0$. For any original weight $\theta^* \in \Sigma$ and error tolerance $\epsilon > 0$, there exists paired mapping functions $g: \Sigma \to \Sigma^{\dagger}$ and $g^{-1}: \Sigma^{\dagger} \to \hat{\Sigma}$, corresponding to Low-Rank Decomposition, Quantization, Pruning, Ergodic Dynamic Systems, or Knowledge Distillation, such that
    \begin{equation}
        \mathfrak{s}(\Sigma) > \mathfrak{s}(\Sigma^{\dagger}) \quad \text{and} \quad  \texttt{dist}(\theta^*, \theta^\dagger)< \epsilon.
    \end{equation}
\end{theorem}

\begin{proof}
    We prove both the set size $\mathfrak{s}(\Sigma) > \mathfrak{s}(\Sigma^{\dagger})$ and the approximation bound $\epsilon$ for each model compression method.

    \subsubsection{Low-Rank Set}
    The compression mapping function $g_{r}: \Sigma \to \Sigma^{\dagger}_r$ is defined by the Singular Value Decomposition (SVD)~\cite{eckart1936approximation}, truncating to rank $r$.
    \begin{itemize}
        \item \textbf{Size Reduction:} The set $\Sigma^{\dagger}_r$ is an algebraic variety of dimension $r(n+m-r)$, which is strictly less than $nm$. Consequently, the Lebesgue measure of the compressed set is zero: $\mathfrak{m}(\Sigma^{\dagger}_r) = 0$.  Thus, $\mathfrak{s}(\Sigma) > \mathfrak{s}(\Sigma^{\dagger}_r) = 0$.
        \item \textbf{Approximation:} By the Eckart-Young-Mirsky theorem~\cite{golub1987generalization}, the approximation error for $\hat{\theta}_r = g^{-1}_{r}(\theta^{\dagger}_r)$ is determined by the tail singular values $\sigma$:
        \begin{equation}
            \texttt{dist}(\theta^*, \theta^{\dagger}_r)= \|\theta^*-g^{-1}(\theta^{\dagger}_r)\|_2 = \sqrt{\sum_{i=r+1}^{\min(n,m)} \sigma_i^{2}}<\epsilon.
        \end{equation}
        As mentioned before, $\theta^*$ is bounded. Hence, the tail sum converges to 0 as $r \to \min(n,m)$. Thus, there exists a rank $r$ such that the error is bounded by $\epsilon$.
    \end{itemize}

    \subsubsection{Quantization Set}
    Let $\theta^* \in \mathbb{R}^D$. The mapping $g_q: \Sigma \to \Sigma^{\dagger}_q$ is the uniform rounding function mapping~\cite{lee2023flexround} onto a discrete lattice with $2^b$ levels over $[-R, R]$. The step size is $\Delta = \frac{2R}{2^b}$.
    \begin{itemize}
        \item \textbf{Size Reduction:} $\texttt{Pointsize}(\Sigma^{\dagger}_q) = b$, which is strictly smaller than the standard 32-bit floating point size of $\Sigma$, ensuring $\mathfrak{s}(\Sigma) > \mathfrak{s}(\Sigma^{\dagger}_q)$.
        \item \textbf{Approximation:} The maximum error occurs when $\theta^*$ is at the center of a lattice hypercube:
        \begin{equation}
             \texttt{dist}(\theta^*,\theta^{\dagger})=\|\theta^* - g^{-1}_q(\theta^{\dagger}_q)\|_2 \le \frac{\Delta}{2} \sqrt{D} = \frac{R\sqrt{D}}{2^b}<\epsilon.
        \end{equation}
        As $b \to \infty$, the term $\frac{R\sqrt{D}}{2^b} \to 0$. For sufficiently large bit-width $b$, the error is $<\epsilon$.
    \end{itemize}

    \subsubsection{Pruning Set}
    The mapping $g_{p}: \Sigma \to \Sigma^{\dagger}_p$ applies Magnitude Pruning~\cite{kohama2023single}, zeroing out all but the $k$ largest coefficients.
    \begin{itemize}
        \item \textbf{Size Reduction:} $\Sigma^{\dagger}_p$ is a set of finite lower-dimensional coordinate subspaces (each of dimension $k < D$). Hence, the $D$-dimensional Lebesgue measure $\mathfrak{m}(\Sigma^{\dagger}_p) = 0$, satisfying $\mathfrak{s}(\Sigma) > \mathfrak{s}(\Sigma^{\dagger}_p)$.
        \item \textbf{Approximation:} The squared approximation error is the energy of the discarded weights:
        \begin{equation}
             \texttt{dist}(\theta^*,\theta^{\dagger})=\|\theta^* - g_{p}^{-1}({\theta^\dagger_p})\|_2^2 = \sum_{j=k+1}^D |\theta^*_{j}|^2<\epsilon,
        \end{equation}
        where $\theta^*_{j}$ denotes the weights sorted by magnitude. As the sparsity parameter $k \to D$, the residual tail sum vanishes to 0.
    \end{itemize}

    \subsubsection{Ergodic Set}
    The mapping $g_{e}: \Sigma \to \Sigma^{\dagger}_{e}$ finds the optimal initial seed $x_0 \in \mathbb{R}^m$ ($m \ll D$) for an ergodic dynamical system $T$ to approximate $\theta^*$.
    \begin{itemize}
        \item \textbf{Size Reduction:} The parameter degrees of freedom are reduced from $D$ to $m$. The intrinsic dimension of $\Sigma^{\dagger}_{erg}$ is bounded by $m$, implying its measure in $\mathbb{R}^D$ is $\mathfrak{m}(\Sigma^{\dagger}_{e}) = 0$. Thus, $\mathfrak{s}(\Sigma) > \mathfrak{s}(\Sigma^{\dagger}_{e})$.
        \item \textbf{Approximation:} We rely on the Topological Transitivity~\cite{akin2012conceptions} of chaotic maps. The \textit{Shadowing Property} guarantees that for open balls $U_i$ centered at $\theta^*_i$ with radius $\frac{\tau}{\sqrt{D}}$, there exists a seed $\theta_0^\dagger$ whose trajectory visits these sets sequentially~\cite{koscielniak2007chaos}. This means that there exits $k$ such that 
        \begin{equation}
        \begin{aligned}
              \texttt{dist}(\theta^*,\theta^{\dagger})&=\|\theta^* - T^k(\theta^\dagger_0)\|_2 \\
              &=\sqrt{\sum_{i=1}^D ({\theta^*_i- (T^k(\theta_0^\dagger))_i})^2}
               < \sqrt{\sum_{i=1}^D \frac{\tau^2}{D}} = \epsilon.
        \end{aligned}
        \end{equation}
    \end{itemize}

    \subsubsection{Knowledge Distillation Set}
    The mapping $g_{k}: \Sigma \to \Sigma^{\dagger}_{k}$ is the distillation process mapping the teacher $\theta^*$ to a smaller student architecture defined on a subspace of dimension $d < D$ (zero-padded).
    \begin{itemize}
        \item \textbf{Size Reduction:} $\Sigma^{\dagger}_{k}$ operates in a $d$-dimensional subspace embedded in $D$ dimensions. Thus, $\mathfrak{m}(\Sigma^{\dagger}_{k}) = 0$, reducing the size functional $\mathfrak{s}(\Sigma) > \mathfrak{s}(\Sigma^{\dagger}_{k})$.
        \item \textbf{Approximation:} According to the KD learning paradigm, the network uses the outputs of the student model $F^\dagger$ to mimic the output of the teacher model $F^*$. We denote the input signal as $x$, and the loss function of two neural networks is equal to the distance function~\eqref{eq:dist2}:
        \begin{equation}
            \mathcal{L}_{kd}= \texttt{dist}(\theta^*, \theta^{\dagger}) = \|F^*(\theta^*,x)-F^\dagger(\theta^\dagger,x)\|_2.
        \end{equation}
By the Universal Approximation Theorem, assuming sufficient student capacity, there exists a configuration $\theta^\dagger$ such that this functional loss is bounded by $\epsilon$.
    \end{itemize}
\end{proof}

\subsection{Theory II: Structural Equivalence}

In this section, we formally prove that five compression methods are mathematically equivalent to specific neural network architectures. 

\begin{theorem}[Structural Equivalence]
    For the five compression methods {low-rank}, {pruning}, {quantization}, {ergodic systems}, and knowledge distillation, there exists a specific neural network compressor $\mathcal{G}_\psi$ with parameter set $\psi$, such that the output set of the network is isomorphic to the compressed set ${\Sigma}^{\dagger}$.
\end{theorem}

\begin{proof}
    We prove this theorem case-by-case. In general, the classical compression mapping $g$ can be reformulated as a neural network compressor $\mathcal{G}$, whose output space constitutes the compressed set $\Sigma^{\dagger}$. An illustrative overview of this equivalence is presented in Figure~\ref{fig:theory2}.

\begin{figure}[ht]
    \centering
    \includegraphics[width=\linewidth]{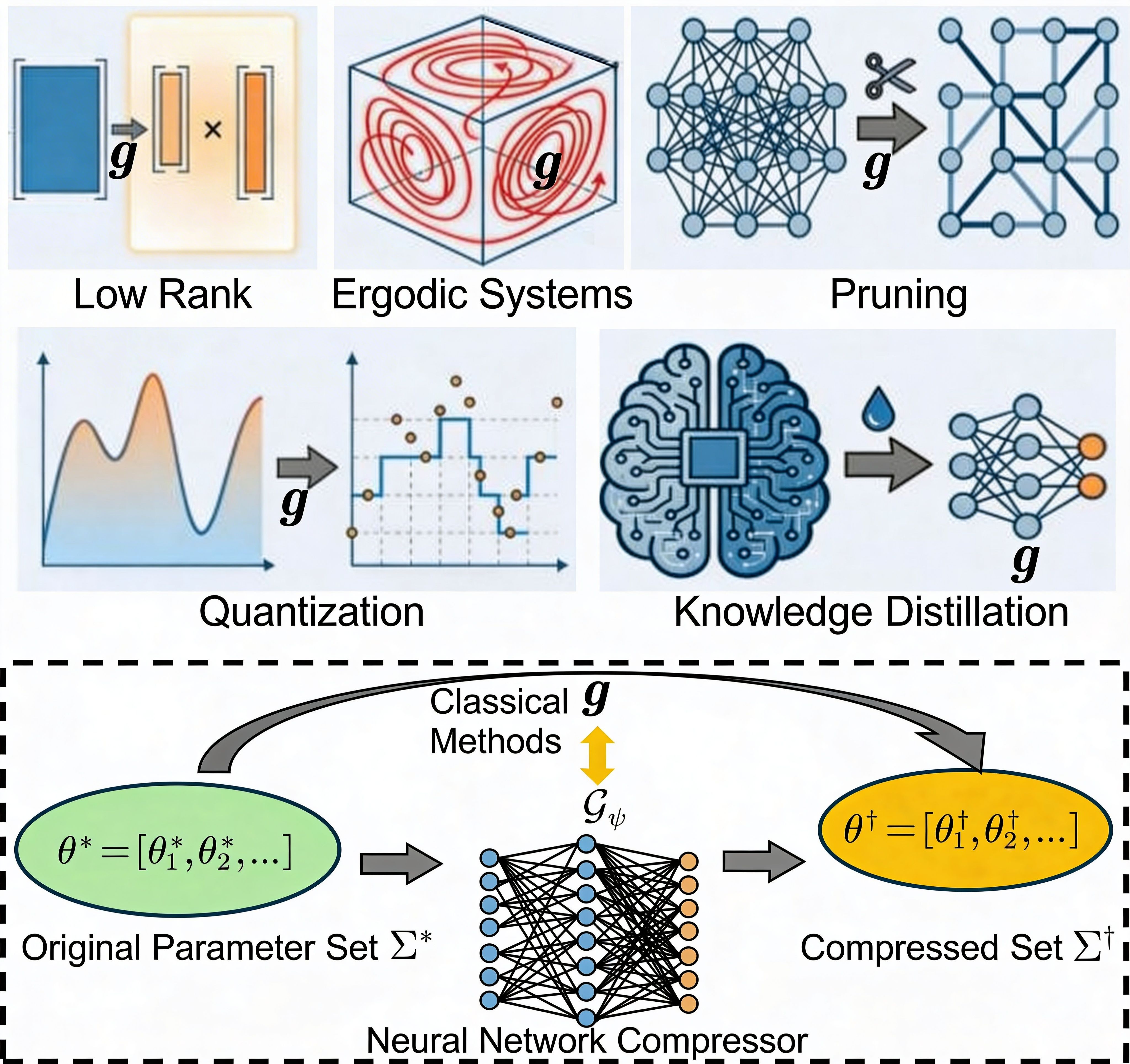}
    \caption{The framework of Theorem 2.}
    \label{fig:theory2}
\end{figure}

\subsubsection{Low-Rank Decomposition as Linear Bottleneck Networks}
Low-rank decomposition approximates the weight matrix $\theta^* \in \mathbb{R}^{n \times m}$ as the product of two smaller matrices $\theta^*  \approx AB$, where $A \in \mathbb{R}^{n \times r}$ and $B \in \mathbb{R}^{r \times m}$ with rank $r \ll \min(n, m)$.

To map this to our framework, we construct a 2-layer linear bottle network $\mathcal{G}_\psi$ with latent input $x = I_{m}$ (the $m\times m$ Identity matrix) and parameter set $\psi = \{{\theta}^\dagger_1, {\theta}^\dagger_2\}$. Here, ${\theta}^\dagger_1 \in \mathbb{R}^{r \times m}$ and ${\theta}^\dagger_2 \in \mathbb{R}^{n \times r}$ correspond to the matrices $B$ and $A$, respectively. The forward pass is defined as a purely linear transformation:
\begin{equation}
    \mathcal{G}_\psi(x) = {\theta}_2^\dagger ({\theta}_1^\dagger I_m) = {\theta}_2^\dagger {\theta}_1^\dagger.
\end{equation}

The compression task becomes the optimization of the following loss function:
\begin{equation}
    \min_{\psi}\mathcal{L}= \min_{\psi} \|\theta^* - \mathcal{G}_\psi(x)\|_2 .
\end{equation}
According to Baldi and Hornik~\cite{baldi1989neural}, the global minimum strictly recovers the optimal rank-$r$ approximation of $W$ given by the Singular Value Decomposition (SVD).

\subsubsection{Quantization as Staircase Activation Networks}
    Classical quantization approximates $\hat{\theta} \approx s \cdot \text{round}(\theta^*/s) + z_{zero}$, with $s$ being scale factor.
    We define this network using a non-linear “staircase" activation function:
    \begin{equation}
        \mathcal{G}_\psi(x) = \psi_{s} \cdot \lfloor x \rfloor + \psi_{zero},
    \end{equation}
    where $\lfloor \cdot \rfloor$ is the rounding operator, the latent code $x = \theta^*/s$ and parameters $\psi = \{s, z_{zero}\}$.

The objective is to minimize the quantization error:
    \begin{equation}
       \min_{\psi}\mathcal{L}= \min_{\psi_s, \psi_{zero}} \| \theta^* - (\psi_{s} \cdot \lfloor x \rfloor + \psi_{zero}) \|_2^2.
    \end{equation}
This formulation strictly restricts the output values to the discrete set $\{ k \cdot \psi_{s} + \psi_{zero} : k \in \mathbb{Z} \}$, which is to approximate the lattice set $\Sigma^*$.

\subsubsection{Pruning as Multiplicative Masking Networks}
    Classical pruning approximates $\hat{\theta} \approx M \odot \theta^*$, where $M_{ij} \in \{0, 1\}$.
    We define the compressor as a Gated Network:
    \begin{equation}
        \mathcal{G}_\psi(x) = \sigma(\psi_{mask}) \odot \theta^*,
    \end{equation}
    where $\psi_{mask} \in \mathbb{R}^d$ are learnable mask scores. 
    
The compression task is to find the optimal binary mask structure that minimizes the approximation error under a sparsity constraint $k$:
    \begin{equation}
    \begin{aligned}
               &\min_{\psi_{mask}}\mathcal{L}=  \min_{\psi_{mask}} \| \theta^* - (\sigma(\psi_{mask}) \odot \theta^*) \|_F^2, \\ &\text{s.t.} \quad \|\sigma(\psi_{mask})\|_0 \le k.
    \end{aligned}
    \end{equation}

If $\sigma(\cdot)$ is the Heaviside step function (or its differentiable surrogate, Gumbel-Softmax), the term $\sigma(\psi_{mask})$ recovers the discrete binary mask $M$, ensuring that the output lies strictly on the pruning set $\hat{\Sigma}_p$.

\subsubsection{Ergodic Systems as Recurrent Neural Networks}
    Classical ergodic compression generates weights via the iteration of a chaotic map: $x_{t+1} = T(x_t); w_t = P(x_t)$.
    We define the compressor $\mathcal{G}_\psi$ as a Recurrent Neural Network (RNN) with the hidden state $h_t \in \mathbb{R}^k$:
    \begin{equation}
        \begin{aligned}
            h_{t+1} &= \sigma(W_{rec} h_t + b_{rec}) \\
            y_{t+1} &= W_{out} h_{t+1} + b_{out}.
        \end{aligned}
    \end{equation}
The initial seed of the trajectory is set to a constant $h_0$. The goal is to learn the ``laws of physics'' (RNN weights) and the ``initial condition'' (seed) that reproduce the target weight vector:
    \begin{equation}
        \min_{\psi}\mathcal{L}=\min_{\psi} \sum_{t=1}^d \| \theta_t^{*} - y_t \|_2^2.
    \end{equation}

By the Universal Approximation Theorem for neural networks, there exists a parameter set $\psi$ such that the RNN dynamics approximate the chaotic map $T$, allowing the compressor to unroll the exact ergodic trajectory.

\subsubsection{Knowledge distillation} Obviously, the paradigm of KD is based on training smaller neural networks. The proof is completed.

\end{proof}

\begin{figure*}[t]
    \centering
    \includegraphics[width=\linewidth]{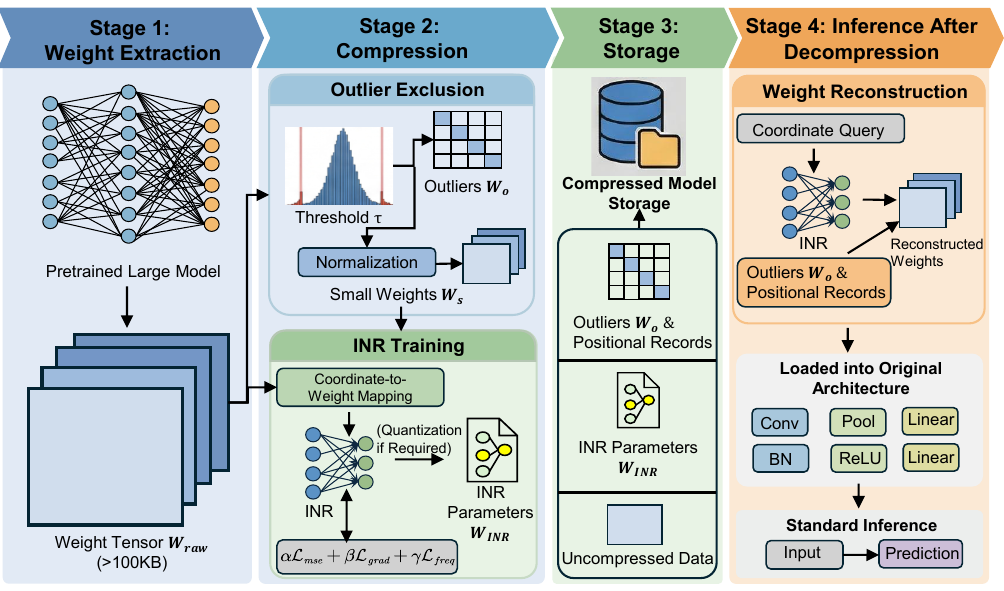}
    \caption{An overview of the proposed \textit{Big2Small} framework. It uses a ``Compression–Decompression'' architecture that encodes discrete weight parameter tensors with lightweight INRs and reconstructs the original weights at inference time.}
    \label{fig:overview}
\end{figure*}

\section{Big2Small}

Our theorem establishes the feasibility that a specific neural network can encapsulate classical model compression methods. This implies the possibility of constructing a small neural network to represent weights of a larger one, termed as \textit{Big2Small}. The \textit{Big2Small} framework introduces a novel system-level approach to neural network compression, which is essentially to translate the idea of INRs from data to a network. Therefore, we can also call our method INR network compressor. As illustrated in Figure~\ref{fig:overview}, \textit{Big2Small} departs from pruning and quantization that directly reduce weights. Instead, it uses a ``Compression–Decompression'' architecture that encodes discrete weight tensors with lightweight INRs. This design substantially reduces storage by saving the encoded representations and reconstructing the original weights at inference time. In doing so, it aims to preserve the network’s salient information, thereby maintaining competitive accuracy relative to traditional compression methods. Consequently, we adopt INRs as our primary compression module, which can be further combined with quantization and other compression methods to enhance compression efficiency.

{The proposed paradigm offers a new direction for model compression. The flexibility of the INR structure allows for unique advantages, such as the ability to iteratively compress the INR itself by other compression methods to achieve higher compression ratios—a capability not completed in traditional methods.} Notably, our method differs from the recently proposed neural-network-based compression approach, Riemannian Neural Dynamics (RieM)~\cite{pei2024data}, which, to our knowledge, is the first to use one neural network to compress the weights of another and achieve satisfactory results. However, RieM focuses on compressing individual weight parameters and does not provide an end-to-end, neural-network-driven compression pipeline.

To be specific, RieM works by regressing each pretrained weight matrix with a separate Riemannian metric and neuron embeddings, using the original weights as reconstruction targets. This multi-stage, hand-engineered pipeline requires extensive hyperparameter tuning (e.g., $D_q$, $D_\mu$, merging thresholds) and substantial offline optimization time per model (from hours to days). It does not provide a unified compressor–decompressor framework that can be reused across architectures or tasks.

\subsection{Training Big2Small}

Let the target model be represented by a set of parameters $\Theta = [ \mathbf{w}_1, \mathbf{w}_2, \dots, \mathbf{w}_N ]$, where each $\mathbf{w}_l$ is a weight tensor of the layer $l$. The INR-reconstructed parameters are $\hat{\Theta} = [ \hat{\mathbf{w}}_1, \hat{\mathbf{w}}_2, \dots, \hat{\mathbf{w}}_N ]$. Now, let us introduce the \textit{Big2Small} framework based on Figure \ref{fig:overview}.

\subsubsection{Outlier-Aware Preprocessing}

The raw weights often contain outliers that are difficult for neural networks to approximate accurately. We therefore introduce a preprocessing strategy prior to INR training.

First, motivated by findings that roughly 1\% of weight outliers dominate quantization error~\cite{zhao2019improving,dettmers2022gpt3}, we locate the indices of the smallest and largest $1\%$ values in $\mathbf{w}_l$:
\begin{equation}
    \mathcal{I}_{\min} = \operatorname{argsort}(\mathbf{w})[:k], \quad \mathcal{I}_{\max} = \operatorname{argsort}(\mathbf{w})[-k:],
\end{equation}
where $k = 1\% \times |\mathbf{w}|$. We store these indices and their exact values in a dictionary. During inference, we patch these values back into the reconstructed tensor, thereby bypassing approximation error for extreme entries.

Second, we normalize the remaining “body" of the distribution to $[-1, 1]$ to match the active regime of the INR used for compression. Let $w_{\min}$ and $w_{\max}$ be the global extrema of $\mathbf{w}$. The normalized weights are defined as
\begin{equation}
    \mathbf{w}_{\mathrm{norm}} = 2 \cdot \frac{\mathbf{w} - w_{\min}}{w_{\max} - w_{\min}} - 1,
\end{equation}
which stabilizes the optimization landscape for the neural compressor.

Finally, we construct a normalized coordinate grid $\mathbf{x}$ that encodes spatial positions. For a tensor with spatial dimensions $(H, W)$, the coordinate for index $(i, j)$ is
\begin{equation}
    \mathbf{x}_{i,j} = \left( 2 \cdot \frac{i}{\max(H-1, 1)} - 1, \; 2 \cdot \frac{j}{\max(W-1, 1)} - 1 \right),
\end{equation}
providing a resolution-agnostic parameterization that enables the INR to learn a continuous function underlying the discrete weight matrix.

\subsubsection{Data-Free Compression Using INRs}

Our objective is to learn a compact representation $\Phi$ with $|\Phi| \ll |\Theta|$ while minimizing reconstruction error.

For a weight tensor $\mathbf{w} \in \mathbb{R}^{d_{\text{out}} \times d_{\text{in}}}$, we view the weights as samples from a continuous signal on a grid. We define a coordinate-to-value mapping (the INR network compressor) $f_\phi$ parameterized by $\phi$, such that
\begin{equation}
    \hat{\mathbf{w}}[\mathbf{x}] = f_\phi(\mathbf{x}).
\end{equation}

For each layer, we train the model by solving
\begin{equation}
    \phi^* = \arg\min_\phi 
    \mathcal{L}\big(\mathbf{w}, f_\phi(\cdot)\big),
\end{equation}
where the loss $\mathcal{L}$ is defined and detailed later.

After INR training, we optionally quantize the INR weights to further reduce storage. For a per-tensor uniform $b$-bit quantizer with scale $s$ and zero-point $z$, we use
\begin{equation}
    q = \operatorname{clip}\Big(\operatorname{round}\big(\tfrac{\hat{\mathbf{w}} - z}{s}\big),\; 0,\; 2^b-1\Big), \quad
    \tilde{\mathbf{w}} = s\, q + z,
\end{equation}
where $s = \tfrac{\max(\hat{\mathbf{w}}) - \min(\hat{\mathbf{w}})}{2^b-1}$ and $z = \min(\hat{\mathbf{w}})$. Here $\tilde{\mathbf{w}}$ denotes the dequantized weights used at inference.

It is noteworthy that we skip batch normalization (BN) layers during compression because their statistics are essential for stable inference. This is actually a common practice in model compression. We compress only convolutional and linear layers larger than 100\,KB to limit overhead and control error. Moreover, if a layer’s reconstruction MSE exceeds 0.01, we discard the compressed version and keep the original.

In summary, the stored model package includes: i) uncompressed parameters; ii) trained INR weights; iii) the outlier dictionary; and iv) position metadata required for inference.

\subsection{The Structure of Implicit Neural Representation}
Our INR models follow the structure in ~\cite{Mehta_2021_ICCV}, employing a dual-MLP encoder for the input coordinates. The overall architecture is illustrated in Figure~\ref{fig:inr}. i) \underline{Synthesis network}: Based on SIREN~\cite{sitzmann2020implicit} with a positional encoder~\cite{Mildenhall2020}, this branch maps input coordinates to reconstructed weight values with strong high-frequency representation capacity. ii) \underline{Modulation network}: A learnable modulator conditions and parameterizes periodic activations, capturing generalizable local functional structure in the weights. 

\begin{figure}[ht]
    \centering
    \includegraphics[width=\linewidth]{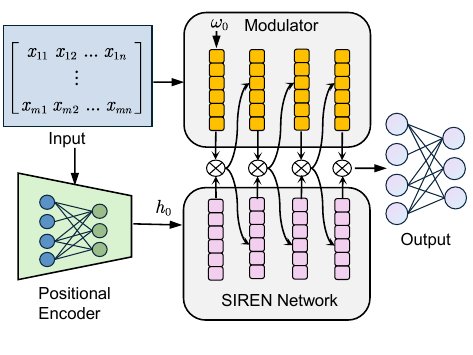}
    \caption{The structure of INR.}
    \label{fig:inr}
\end{figure}

To further enhance high-frequency fidelity, we incorporate a frequency-aware loss function that penalizes first-order gradient discrepancies and reweights residuals in the Fourier domain, encouraging accurate reconstruction of both low- and high-frequency components.

\subsubsection{Frequency-Modulated INR}

Given an input coordinate vector $\mathbf{x} \in \mathbb{R}^{d_{in}}$, the modulator $\mathcal{M}$ and the positional encoder $\mathcal{P}$ simultaneously map the input via a dual-MLP:

\begin{equation}
\begin{cases}
& \mathbf{h}_0 = \mathrm{Linear}_{\text{head}}\big(\mathcal{P}(\mathbf{x})\big),\\
& \mathcal{M}(\mathbf{x}) = \mathrm{Linear}(\mathbf{x}).
\end{cases}
\end{equation}
Next, the modulator output with a base frequency $\omega_{0}$ produces a spatially varying frequency vector $\mathbf{p}(\mathbf{x})$:
\begin{equation}
    \mathbf{p}(\mathbf{x}) = \mathcal{M}(\mathbf{x}) + \omega_{0}.
\end{equation}
For the weight matrix $\mathbf{W}_i$ and bias $\mathbf{b}_i$ in the $i$-th hidden layer, the output is:
\begin{equation}
    \mathbf{h}_{i+1} = \sin\!\left( \mathbf{p}(\mathbf{x}) \odot (\mathbf{W}_i \mathbf{h}_i + \mathbf{b}_i) \right),
\end{equation}
where $\odot$ denotes element-wise multiplication. 

The modulator predicts a location-dependent frequency map that dynamically scales the periodicity of sinusoidal activations in the synthesis branch, controlling spectral bias locally~\cite{Mehta_2021_ICCV}. Unlike SIREN with the fixed $\sin\big(\omega_0(\mathbf{W}\mathbf{h}+\mathbf{b})\big)$, the effective frequency is a learned function of the input. The final hidden state is mapped by a linear tail to the target shape:
\begin{equation}
    \hat{\mathbf{w}} = \mathrm{Linear}_{\text{tail}}(\mathbf{h}_{\text{final}}).
\end{equation}

The positional encoder projects the inputs to a Fourier feature space to efficiently capture higher-frequency structure~\cite{tancik2020fourier}. For the frequency band $i$ and the scalar input $x_j$
\begin{equation}
    \gamma_{i,j}(x_j) = \big[ \sin(\sigma^i \pi x_j),\; \cos(\sigma^i \pi x_j) \big],
\end{equation}
with $i \in \{0, \dots, L-1\}$, $L$  band length, and scaling factor $\sigma$. The encoder concatenates the raw input and all sinusoidal expansions:

\begin{equation}
    \mathcal{P}(\mathbf{x}) = \big[ \mathbf{x},\; \gamma_{0}(\mathbf{x}),\; \gamma_{1}(\mathbf{x}),\; \dots,\; \gamma_{L-1}(\mathbf{x}) \big].
\end{equation}

To prevent aliasing, we apply a Nyquist bandwidth limiter that chooses the maximum band via
\begin{equation}
    L \approx \big\lfloor \log_2(\text{Nyquist Rate}) \big\rfloor = \big\lfloor \log_2(S/2) \big\rfloor,
\end{equation}
where $S$ denotes spatial resolution.

Finally, we size the INR by the target compression ratio $r$. With $P$ raw parameters, the base hidden width is controlled by $h = \max\big(32, \sqrt{P/(4r)}\big)$, and we use six hidden layers arranged as $[h, h, h, h/2, h/2, h/2]$.

\subsubsection{Frequency-Aware Loss Function}
Although pre-trained weights appear near-Gaussian distribution, the INR reconstructions may underfit higher frequencies. We therefore adopt a frequency-aware objective to encourage low- and high-frequency recovery. The total loss combines MSE, Gradient Difference (Grad)~\cite{mathieu2015deep}, and Focal Frequency (Freq)~\cite{jiang2021focal} terms:
\begin{equation}
    \mathcal{L}_{\text{total}} = \alpha \, \mathcal{L}_{\text{mse}} + \beta \, \mathcal{L}_{\text{grad}} + \psi \, \mathcal{L}_{\text{freq}},
    \label{eq:loss}
\end{equation}
with $\alpha{=}1,\; \beta{=}0.5,\; \psi{=}0.1$.

First, the Grad loss is primarily used in video prediction to reduce blurring at boundaries~\cite{mathieu2015deep}. For a 2D weight $\mathbf{w}$, discrete gradients along horizontal ($h$) and vertical ($v$) directions are
\begin{equation}
    \nabla_h \mathbf{w} = \mathbf{w}_{i, j+1} - \mathbf{w}_{i, j}, \quad \nabla_v \mathbf{w} = \mathbf{w}_{i+1, j} - \mathbf{w}_{i, j},
\end{equation}
and the gradient loss minimizes first-order differences:
\begin{equation}
    \mathcal{L}_{\text{grad}} = \big\| \nabla_h \hat{\mathbf{w}} - \nabla_h \mathbf{w} \big\|_2^2 + \big\| \nabla_v \hat{\mathbf{w}} - \nabla_v \mathbf{w} \big\|_2^2.
\end{equation}

Second, the Freq loss was first introduced in the image reconstruction task to address the high-frequency periodic patterns missing issue~\cite{jiang2021focal}. The frequency loss operates in the Fourier domain via a 2D FFT:
\begin{equation}
    \mathcal{L}_{\text{freq}} = \big\| \mathrm{FFT}(\hat{\mathbf{w}}) - \mathrm{FFT}(\mathbf{w}) \big\|_2^2.
\end{equation}
Optimizing in frequency space penalizes global high-frequency discrepancies that may appear as small local errors in the spatial domain.

\section{Experiments}
\subsection{Experimental Configuration} 

In this experiment, we compress some mainstream image recognition models, such as UNet~\cite{ronneberger2015u}, ResNet~\cite{he2015deep}, and Swin-Transformer~\cite{liu2021swin,liu2022swin}, and validate their performance on the image classification and segmentation tasks. These pretrained models are collected from the PyTorch Image Models (timm) repository~\cite{rw2019timm}. Several DFMC methods are utilized as the baselines. We use their default hyperparameter settings in the experiments.

For our method, we set the initial learning rate to $10^{-4}$ with the AdamW optimizer. The learning rate scheduler is CosineAnnealingLR and $T_{max}=200$. All INRs are executed 10,000 epochs to ensure convergence. To compare with other baselines, we compress the model by a ratio of 1.5 in 32-bit full-precision training and quantize it to lower bits.

\subsection{Model Compression Results}
\textbf{Image Classification on ImageNet~\cite{deng2009imagenet}}. Experimental results are presented in Table \ref{tab:img}. The Big2Small demonstrates competitive performance across different architectures when compared to other compression techniques. For ResNet18, Big2Small-6bit achieves a compression ratio of 5.9 with a model size of only 7.6 MB, outperforming DSG and Squant while maintaining a higher compression rate. In the ResNet50 experiments, Big2Small-8bit notably secures the highest Top-1 accuracy of 73.98\% among the compressed models, surpassing UDFC and RieM despite using a slightly lower compression ratio of 7.7. Furthermore, for Swin Transformer, our method also achieves the best accuracy in the same model size level. Compared to conventional quantization methods that only compress models to a fixed size, the Big2Small is more flexible. The original model can first be compressed into the format of 32bit and then quantized to a smaller size.

\begin{table}[ht]
\centering
\caption{Comparison of different compression methods for three notable image recognition networks on the ImageNet dataset. W/A means the numbers of weight/activation bits.}
\scalebox{0.85}{
\begin{tabular}{l|cccc}
\hline
                & Size (MB)                    & W/A Bits                      & Compression Ratio           & Top1 Acc (\%)                          \\ \hline
ResNet18~\cite{he2015deep}        & 44.6                         & 32/32                        & -                           & 71.47                                  \\
DSG~\cite{zhang2021diversifying}             & \cellcolor[HTML]{DAE8FC}8.4  & \cellcolor[HTML]{DAE8FC}6/6  & \cellcolor[HTML]{DAE8FC}5.3 & \cellcolor[HTML]{DAE8FC}70.46          \\
Squant~\cite{guo2022squant}          & \cellcolor[HTML]{DAE8FC}8.4  & \cellcolor[HTML]{DAE8FC}6/6  & \cellcolor[HTML]{DAE8FC}5.3 & \cellcolor[HTML]{DAE8FC}70.74          \\
UDFC~\cite{bai2023unified}            & \cellcolor[HTML]{DAE8FC}8.4  & \cellcolor[HTML]{DAE8FC}6/6  & \cellcolor[HTML]{DAE8FC}5.3 & \cellcolor[HTML]{DAE8FC}\textbf{72.76} \\
RieM~\cite{pei2024data}            & \cellcolor[HTML]{DAE8FC}8.4  & \cellcolor[HTML]{DAE8FC}8/16 & \cellcolor[HTML]{DAE8FC}5.3 & \cellcolor[HTML]{DAE8FC}71.80          \\
Big2Small-32bit & 30.4                         & 32/32                        & 1.5                         & 71.85                                  \\
Big2Small-16bit & 15.2                         & 16/16                        & 2.9                         & 71.54                                  \\
Big2Small-6bit  & \cellcolor[HTML]{DAE8FC}7.6  & \cellcolor[HTML]{DAE8FC}6/6  & \cellcolor[HTML]{DAE8FC}5.9 & \cellcolor[HTML]{DAE8FC}71.24          \\ \hline
ResNet50~\cite{he2015deep}        & 97.5                         & 32/32                        & -                           & 77.72                                  \\
GDFQ~\cite{xu2020generative}            & \cellcolor[HTML]{DAE8FC}12.3 & \cellcolor[HTML]{DAE8FC}4/4  & \cellcolor[HTML]{DAE8FC}7.9 & \cellcolor[HTML]{DAE8FC}55.65          \\
Squant~\cite{guo2022squant}           & \cellcolor[HTML]{DAE8FC}12.3 & \cellcolor[HTML]{DAE8FC}4/4  & \cellcolor[HTML]{DAE8FC}7.9 & \cellcolor[HTML]{DAE8FC}70.80          \\
UDFC~\cite{bai2023unified}            & \cellcolor[HTML]{DAE8FC}12.3 & \cellcolor[HTML]{DAE8FC}4/4  & \cellcolor[HTML]{DAE8FC}7.9 & \cellcolor[HTML]{DAE8FC}72.09          \\
RieM~\cite{pei2024data}            & \cellcolor[HTML]{DAE8FC}12.3 & \cellcolor[HTML]{DAE8FC}4/4  & \cellcolor[HTML]{DAE8FC}7.9 & \cellcolor[HTML]{DAE8FC}73.26          \\
Big2Small-32bit & 50.3                         & 32/32                        & 1.9                         & 76.54                                  \\
Big2Small-16bit & 25.2                         & 16/16                        & 3.9                         & 74.33                                  \\
Big2Small-8bit  & \cellcolor[HTML]{DAE8FC}12.6 & \cellcolor[HTML]{DAE8FC}8/8  & \cellcolor[HTML]{DAE8FC}7.7 & \cellcolor[HTML]{DAE8FC}\textbf{73.98} \\ \hline
Swin-T~\cite{liu2022swin}          & 116.0                        & 32/32                        & -                           & 81.35                                  \\
PSAQ-ViT~\cite{li2022patch}        & \cellcolor[HTML]{DAE8FC}14.5 & \cellcolor[HTML]{DAE8FC}4/8  & \cellcolor[HTML]{DAE8FC}8.0 & \cellcolor[HTML]{DAE8FC}71.79          \\
PSAQ-ViT V2\cite{li2023psaq}     & \cellcolor[HTML]{DAE8FC}14.5 & \cellcolor[HTML]{DAE8FC}4/8  & \cellcolor[HTML]{DAE8FC}8.0 & \cellcolor[HTML]{DAE8FC}76.28          \\
RieM~\cite{pei2024data}            & \cellcolor[HTML]{DAE8FC}14.5 & \cellcolor[HTML]{DAE8FC}4/8  & \cellcolor[HTML]{DAE8FC}8.0 & \cellcolor[HTML]{DAE8FC}76.30          \\
PSAQ-ViT V2~\cite{li2023psaq}     & 29.0                         & 8/8                          & 4.0                         & 80.21                                  \\
RieM~\cite{pei2024data}            & 29.0                         & 8/8                          & 4.0                         & 80.85                                  \\
Big2Small-32bit & 77.3                         & 32/32                        & 1.5                         & 81.02                                  \\
Big2Small-16bit & 38.6                         & 16/16                        & 3.0                         & 80.88                                  \\
Big2Small-6bit  & \cellcolor[HTML]{DAE8FC}14.6 & \cellcolor[HTML]{DAE8FC}6/6  & \cellcolor[HTML]{DAE8FC}7.9 & \cellcolor[HTML]{DAE8FC}\textbf{77.25} \\ \hline
Swin-S~\cite{liu2022swin}          & 200.0                        & 32/32                        & -                           & 83.20                                  \\
PSAQ-ViT~\cite{li2022patch}        & \cellcolor[HTML]{DAE8FC}25.0 & \cellcolor[HTML]{DAE8FC}4/8  & \cellcolor[HTML]{DAE8FC}8.0 & \cellcolor[HTML]{DAE8FC}75.14          \\
PSAQ-ViT V2~\cite{li2023psaq}     & \cellcolor[HTML]{DAE8FC}25.0 & \cellcolor[HTML]{DAE8FC}4/8  & \cellcolor[HTML]{DAE8FC}8.0 & \cellcolor[HTML]{DAE8FC}78.86          \\
RieM~\cite{pei2024data}            & \cellcolor[HTML]{DAE8FC}25.0 & \cellcolor[HTML]{DAE8FC}4/8  & \cellcolor[HTML]{DAE8FC}8.0 & \cellcolor[HTML]{DAE8FC}79.84          \\
PSAQ-ViT V2~\cite{li2023psaq}     & 50.0                         & 8/8                          & 4.0                         & 82.13                                  \\
RieM~\cite{pei2024data}            & 50.0                         & 8/8                          & 4.0                         & 82.96                                  \\
Big2Small-32bit & 142.9                        & 32/32                        & 1.4                         & 83.04                                  \\
Big2Small-16bit & 71.5                         & 16/16                        & 2.8                         & 82.88                                  \\
Big2Small-6bit  & \cellcolor[HTML]{DAE8FC}26.8 & \cellcolor[HTML]{DAE8FC}6/6  & \cellcolor[HTML]{DAE8FC}7.5 & \cellcolor[HTML]{DAE8FC}\textbf{80.35} \\ \hline
\end{tabular}}
\label{tab:img}
\end{table}

\begin{figure*}[ht]
    \centering
    \includegraphics{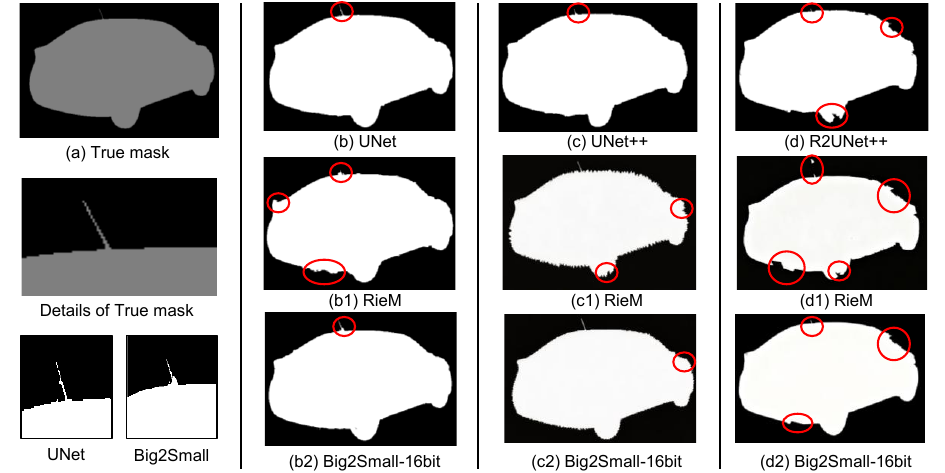}
    \caption{Visualizing segmentation results of the original UNet models and the compressed ones.}
    \label{fig:car}
\end{figure*}

\textbf{Image Segmentation on Carvana}. This dataset is obtained from the Carvana Image Masking Challenge~\footnote{\url{https://www.kaggle.com/competitions/carvana-image-masking-challenge/overview}}. The Carvana company collected images of 318 unique vehicles, each with 16 standard rotating images of a high resolution 1,918 $\times$ 1,080. The image segmentation results are illustrated in Table~\ref{tab:car}. We apply UNet, UNet++, and R2UNet as foundation models for compression and adopt RieM as the baseline method. The proposed Big2Small method demonstrates superior performance trade-offs compared to the RieM baseline. Specifically, the 32-bit Big2Small variant consistently achieves the best performance among the compressed models, effectively preserving the predictive capability of the original networks while offering substantial size reduction. For instance, in the UNet configuration, Big2Small-32bit reduces the model size to 22.9 MB while maintaining a competitive mIOU of 95.32\%, outperforming the RieM method which yields 94.67\% mIOU with 25.8MB. Furthermore, the 16-bit Big2Small implementation offers even greater compression rates—reducing the UNet++ model to just 3.1 MB—while frequently surpassing or matching the accuracy metrics of the corresponding 16-bit RieM counterparts.

\begin{table}[ht]
\centering
\caption{Comparison of different compression methods on the Carvana dataset. Where W/A means the numbers of weight/activation bits.}
\begin{tabular}{lcccc}
\hline
                & Size (MB)                     & W/A bits                      & mACC(\%)                               & mIOU(\%)                               \\ \hline
UNet~\cite{ronneberger2015u}            & 51.5                         & 32/32                         & 99.52                                  & 95.87                                  \\
RieM~\cite{pei2024data}            & \cellcolor[HTML]{DAE8FC}25.8 & \cellcolor[HTML]{DAE8FC}16/16 & \cellcolor[HTML]{DAE8FC}98.52          & \cellcolor[HTML]{DAE8FC}94.67          \\
Big2Small-32bit & \cellcolor[HTML]{DAE8FC}22.9 & \cellcolor[HTML]{DAE8FC}32/32 & \cellcolor[HTML]{DAE8FC}\textbf{99.07} & \cellcolor[HTML]{DAE8FC}\textbf{95.32} \\
Big2Small-16bit & \cellcolor[HTML]{DAE8FC}11.5 & \cellcolor[HTML]{DAE8FC}16/16 & \cellcolor[HTML]{DAE8FC}98.75          & \cellcolor[HTML]{DAE8FC}94.86          \\ \hline
UNet++~\cite{zhou2018unet++}          & 9.2                          & 32/32                         & 99.01                                  & 95.62                                  \\
RieM~\cite{pei2024data}            & \cellcolor[HTML]{DAE8FC}4.6  & \cellcolor[HTML]{DAE8FC}16/16 & \cellcolor[HTML]{DAE8FC}97.15          & \cellcolor[HTML]{DAE8FC}90.35          \\
Big2Small-32bit & \cellcolor[HTML]{DAE8FC}6.1  & \cellcolor[HTML]{DAE8FC}32/32 & \cellcolor[HTML]{DAE8FC}\textbf{98.01} & \cellcolor[HTML]{DAE8FC}\textbf{93.24} \\
Big2Small-16bit & \cellcolor[HTML]{DAE8FC}3.1  & \cellcolor[HTML]{DAE8FC}16/16 & \cellcolor[HTML]{DAE8FC}95.11          & \cellcolor[HTML]{DAE8FC}92.74          \\ \hline
R2UNet~\cite{alom2018recurrent}          & 39.1                         & 32/32                         & 99.08                                  & 95.41                                  \\
RieM~\cite{pei2024data}            & \cellcolor[HTML]{DAE8FC}19.6 & \cellcolor[HTML]{DAE8FC}16/16 & \cellcolor[HTML]{DAE8FC}97.12          & \cellcolor[HTML]{DAE8FC}93.53          \\
Big2Small-32bit & \cellcolor[HTML]{DAE8FC}26.1 & \cellcolor[HTML]{DAE8FC}32/32 & \cellcolor[HTML]{DAE8FC}\textbf{97.95} & \cellcolor[HTML]{DAE8FC}\textbf{94.35} \\
Big2Small-16bit & \cellcolor[HTML]{DAE8FC}13.0 & \cellcolor[HTML]{DAE8FC}16/16 & \cellcolor[HTML]{DAE8FC}96.12          & \cellcolor[HTML]{DAE8FC}94.59          \\ \hline
\end{tabular}
\label{tab:car}
\end{table}

Furthermore, we visualize the segmentation results of different methods in Figure~\ref{fig:car}. The antenna presents a challenging feature to segment clearly. While the baseline method RieM fails to capture this structure entirely, \textit{Big2Small} successfully detects this fine detail, closely matching the performance of the raw UNet, albeit with minor noise. This highlights the superior reconstruction fidelity of our method in image segmentation tasks over RieM.

\subsection{Analysis Experiments}
\textbf{Weight Reconstruction Evaluation}. Here, we compare the weight reconstruction quality with RieM~\cite{pei2024data}. In this experiment, we adopt a convolutional weight of ResNet50 to train \textit{Big2Small} and RieM models. The reconstruction quality is evaluated by weight distributions and Quantile-Quantile (Q–Q) plots. Specifically, Q-Q plots compare two probability distributions by plotting their quantiles against one another, offering a precise visual tool to detect deviations like skewness or heavy tails. Mathematically, for a set of $n$ sorted observations $x_1 \le x_2 \le \dots \le x_n$, the theoretical quantile $q_i$ of a standard normal distribution is approximated via the inverse cumulative distribution function (CDF) $\Phi^{-1}$:
\begin{equation}
    q_i = \Phi^{-1}\left(\frac{i - 0.5}{n}\right),
\end{equation}
where every point on the graph represents a pair $(q_i, x_i)$.

The results is illustrated in Figure \ref{fig:weight-re}. First, for the weight distribution, the \textit{Big2Small} accurately preserves the shape, peak, and spread of the original weight distribution. In contrast, the RieM curve appears to deviate from the original density. Second, in the Q-Q plot, while the RieM exhibits a visible mismatch in lower and higher frequencies, the \textit{Big2Small} data points align tightly with the “Perfect Match" diagonal line. This indicates that \textit{Big2Small} successfully maintains the statistical properties of the original weight distribution across the entire range of values, ensuring a high-fidelity reconstruction that the baseline method fails to achieve.

Moreover, we compare the performance of two methods in varying compression ratios. The MSE and cosine similarity are utilized for quantitative evaluation, as depicted in Figure~\ref{fig:weight-comprare}. \textit{Big2Small} consistently outperforms the baseline across all tested compression levels. In terms of the reconstruction error, \textit{Big2Small} maintains significantly lower MAE values, indicating a higher precision in approximating the original weight magnitudes. Complementing this, the Cosine Similarity analysis shows that \textit{Big2Small} achieves scores closer to 1.0, evidencing superior preservation of the weight vectors' directional properties.

\begin{figure}[ht]
    \centering
    \includegraphics{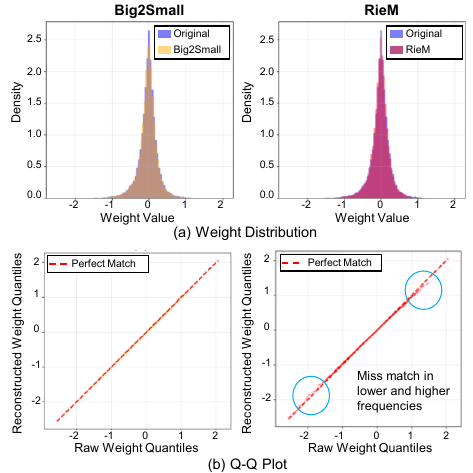}
    \caption{Comparison of weight reconstruction qualities on the convolutional layer of ResNet50.}
    \label{fig:weight-re}
    \vspace{-0.4cm}
\end{figure}

\begin{figure}[ht]
    \centering
    \includegraphics{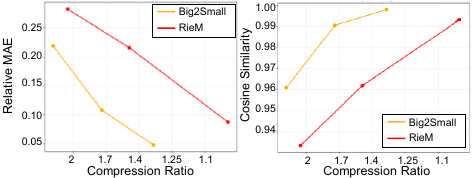}
    \caption{Reconstruction error vs compression ratios on the convolutional layer of ResNet50.}
    \label{fig:weight-comprare}
\end{figure}

\textbf{Ablation Study}. We validate the proposed strategies of training INRs in Table \ref{tab:ablation}. Clearly, superior performance is achieved across both convolutional and linear architectures by using outlier-aware preprocessing and frequency-aware loss. After preprocessing, \textit{Big2Small} shows significant improvements in terms of MAE, cosine similarity, and PSNR.  Similarly, our loss function yields the best results across all metrics, surpassing standard MSE, MSE+Grad, and MSE+Freq variations, confirming the efficacy of the proposed technique.

\begin{table}[ht]
\centering
    \caption{Ablation study of the Big2Small using ResNet50. ``w/o pr'' denotes without using our preprocessing strategy; ``Rel MAE'' is relative MAE; ``Cosine'' represents cosine similarity.}
\begin{tabular}{ll|ccc}
\hline
Layer                         & Strategy                         & Rel MAE                                 & Cosine                                 & PSNR                                   \\ \hline
                              & \cellcolor[HTML]{DAE8FC}w/o pre  & \cellcolor[HTML]{DAE8FC}0.0832          & \cellcolor[HTML]{DAE8FC}0.953          & \cellcolor[HTML]{DAE8FC}29.32          \\
                              & \cellcolor[HTML]{DAE8FC}pre      & \cellcolor[HTML]{DAE8FC}\textbf{0.0085} & \cellcolor[HTML]{DAE8FC}\textbf{1.000} & \cellcolor[HTML]{DAE8FC}\textbf{54.21} \\
                              & \cellcolor[HTML]{FFFFC7}MSE      & \cellcolor[HTML]{FFFFC7}0.0672          & \cellcolor[HTML]{FFFFC7}0.996          & \cellcolor[HTML]{FFFFC7}34.29          \\
                              & \cellcolor[HTML]{FFFFC7}MSE+Grad & \cellcolor[HTML]{FFFFC7}0.0421          & \cellcolor[HTML]{FFFFC7}0.996          & \cellcolor[HTML]{FFFFC7}35.21          \\
                              & \cellcolor[HTML]{FFFFC7}MSE+Freq & \cellcolor[HTML]{FFFFC7}0.0351          & \cellcolor[HTML]{FFFFC7}0.996          & \cellcolor[HTML]{FFFFC7}35.52          \\
\multirow{-6}{*}{Convolution} & \cellcolor[HTML]{FFFFC7}Ours     & \cellcolor[HTML]{FFFFC7}\textbf{0.0085} & \cellcolor[HTML]{FFFFC7}\textbf{1.000} & \cellcolor[HTML]{FFFFC7}\textbf{54.21} \\ \hline
                              & \cellcolor[HTML]{DAE8FC}w/o pre  & \cellcolor[HTML]{DAE8FC}0.0927          & \cellcolor[HTML]{DAE8FC}0.923          & \cellcolor[HTML]{DAE8FC}24.38          \\
                              & \cellcolor[HTML]{DAE8FC}pre      & \cellcolor[HTML]{DAE8FC}\textbf{0.0496} & \cellcolor[HTML]{DAE8FC}\textbf{0.999} & \cellcolor[HTML]{DAE8FC}\textbf{37.63} \\
                              & \cellcolor[HTML]{FFFFC7}MSE      & \cellcolor[HTML]{FFFFC7}0.1032          & \cellcolor[HTML]{FFFFC7}0.906          & \cellcolor[HTML]{FFFFC7}25.32          \\
                              & \cellcolor[HTML]{FFFFC7}MSE+Grad & \cellcolor[HTML]{FFFFC7}0.0785          & \cellcolor[HTML]{FFFFC7}0.981          & \cellcolor[HTML]{FFFFC7}29.26          \\
                              & \cellcolor[HTML]{FFFFC7}MSE+Freq & \cellcolor[HTML]{FFFFC7}0.0674          & \cellcolor[HTML]{FFFFC7}0.983          & \cellcolor[HTML]{FFFFC7}30.27          \\
\multirow{-6}{*}{Linear}      & \cellcolor[HTML]{FFFFC7}Ours     & \cellcolor[HTML]{FFFFC7}\textbf{0.0496} & \cellcolor[HTML]{FFFFC7}\textbf{0.998} & \cellcolor[HTML]{FFFFC7}\textbf{37.63} \\ \hline
\end{tabular}
\label{tab:ablation}
\end{table}

Lastly, to illustrate the effect of the loss function, Figure~\ref{fig:qq-comprare} compares reconstructions trained with pure MSE versus the proposed frequency-aware loss using Q-Q plots. The results demonstrate that our method achieves superior distributional alignment, as evidenced by data points following the diagonal ``Perfect Match" line $(y=x)$ more closely than the MSE Loss.


\begin{figure}[ht]
    \centering
    \includegraphics[width=\linewidth]{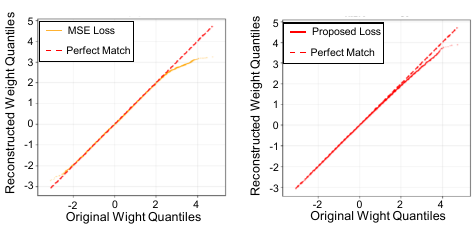}
    \caption{The Q-Q plot of using different loss functions to reconstruct the weight of the linear layer.}
    \label{fig:qq-comprare}
\end{figure}

\textbf{Combination with Other Compression Methods}. Since \textit{Big2Small} operates as a neural network architecture, its efficiency can be further enhanced by integrating orthogonal model compression techniques. To evaluate this potential, we applied Low-Rank Decomposition (SVD), Pruning (P) with 60\% pruning rate, and Quantization (Q) to the \textit{Big2Small} framework. The results, summarized in Table \ref{tab:b2s}, demonstrate that \textit{Big2Small} is highly compatible with traditional compression methods, achieving significantly higher compression ratios with minimal performance degradation. Notably, 6-bit Quantization (Big2Small+Q) emerged as the most effective combination, achieving a 5.9$\times$ compression ratio while maintaining a competitive 71.24\% accuracy. While SVD and pruning also show substantial size reductions, they incur slightly higher accuracy trade-offs. These preliminary findings highlight the extensibility of \textit{Big2Small}, and we intend to explore optimized hybrid compression strategies in future work.

\begin{table}[ht]
\centering
\caption{Compressing Big2Small on the ImageNet dataset. Where P denotes pruning and Q denotes quantization.}
\begin{tabular}{l|ccc}
\hline
                   & \multicolumn{1}{l}{Size (MB)} & \multicolumn{1}{l}{Compression Ratio} & \multicolumn{1}{l}{Top1 Acc (\%)} \\ \hline
ResNet18           & 44.59                         &     -                                  & 71.47                             \\
Big2Small          & 30.4                          & 1.5                                   & 71.85                             \\
Big2Small+SVD      & 10.32                         & 2.9                                   & 70.32                             \\
Big2Small+P (60\%) & 9.12                          & 4.9                                   & 69.82                             \\
Big2Small+Q (6bit) & 7.59                          & 5.9                                   & 71.24                             
\\ \hline
\end{tabular}
\label{tab:b2s}
\end{table}

\textbf{Inference Time}. We evaluate the inference latency of \textit{Big2Small} against the original model on the ImageNet dataset using an NVIDIA RTX PRO 6000 GPU. The results are illustrated in Figure~\ref{fig:tp}. During the inference stage, \textit{Big2Small} must employ all INRs to reconstruct the model weights. Consequently, an increase in computational complexity and time is inevitable. \textit{Big2Small} increases inference latency by approximately 30\% compared to the raw model. However, this overhead remains acceptable for non-real-time scenarios where storage efficiency is the primary constraint, such as deployments on edge device.


\begin{figure}[ht]
    \centering
    \includegraphics[width=\linewidth]{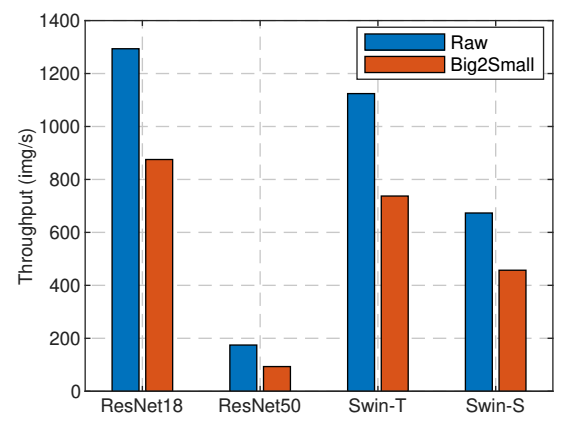}
    \caption{The throughput of raw and Big2Small models on the ImageNet test set. }
    \label{fig:tp}
\end{figure}

\textbf{Comparison with Other Encoders}. \textit{Big2Small} serves as a general framework capable of integrating various encoders and generative models. In this experiment, we evaluate the suitability of different models, such as SIREN~\cite{sitzmann2020implicit}, VAE~\cite{kingmaauto}, and GAN~\cite{goodfellow2020generative}, by reconstructing a convolutional weight tensor from ResNet-50. As shown in Figure~\ref{fig:encoder}, our method outperforms these baselines. Notably, GANs are relatively unsuitable for this task, exhibiting a significantly high relative MAE. While VAE and SIREN yield competitive results, they remain inferior to our approach. This indicates that our specific INR structure possesses superior representational capacity, making it the optimal choice for DFMC.

\begin{figure}[ht]
    \centering
    \includegraphics{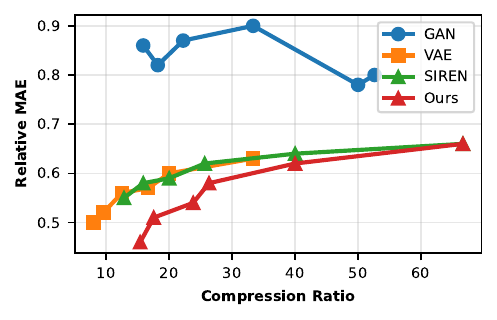}
    \caption{Comparison of different encoders for compressing a convolution weight of ResNet50.}
    \label{fig:encoder}
\end{figure}

\section{Conclusion}

In this paper, we have established a comprehensive theoretical framework that unifies various post-training model compression techniques and redefines them as neural function approximation problems. Building upon this mathematical and structural equivalence, we have introduced \textit{Big2Small}, a novel data-free model compression (DFMC) paradigm that utilizes Implicit Neural Representations (INRs) to compress large deep learning models. By treating model weights as discrete data, we have demonstrated that small neural networks can effectively encode the parameters of larger architectures without requiring access to original or synthetic data. Experimental validations have confirmed that \textit{Big2Small} achieves highly competitive performance and efficiency across multiple architectures. Future work will focus on optimizing the training efficiency and exploring higher-capacity representation networks, such as polynomial networks, to further improve reconstruction performance.



\bibliographystyle{ieeetr}
\bibliography{ref.bib}





\end{document}